\documentclass{article}

\usepackage{PRIMEarxiv}

\usepackage[utf8]{inputenc} 
\usepackage[T1]{fontenc}    
\usepackage{hyperref}       
\usepackage{url}            
\usepackage{booktabs}       
\usepackage{amsfonts}       
\usepackage{nicefrac}       
\usepackage{microtype}      
\usepackage{lipsum}
\usepackage{fancyhdr}       
\usepackage{graphicx}       
\graphicspath{{media/}}     

\usepackage{soul}
\usepackage{url}
\usepackage{amssymb}
\usepackage[small]{caption}
\usepackage{graphicx}
\usepackage{amsmath}
\usepackage{amsthm}
\usepackage{tikz}
\usepackage{enumitem}

\usepackage{times}
\usepackage[switch]{lineno}
\usepackage{latexsym}
\usepackage[geometry]{ifsym}
\usepackage[noend]{algpseudocode}
\usepackage{algorithm}
\usepackage{color}
\algdef{SE}[DOWHILE]{Do}{doWhile}{\algorithmicdo}[1]{\algorithmicwhile\ #1}%
\usepackage{pgfplots}
\usepgfplotslibrary{fillbetween}
\usetikzlibrary{patterns}
\usepackage{subcaption}
\usepackage{times}

\newcommand{\nameHbs}{\mathtt{Hbs}}
\newcommand{\nameMax}{\mathtt{Max}}
\newcommand{\nameCar}{\mathtt{Car}}

\newcommand{\Hbs}{\sigma^{\nameHbs}}
\newcommand{\Max}{\sigma^{\nameMax}}
\newcommand{\Car}{\sigma^{\nameCar}}

\newcommand{\AS}{\mathbf{AF}}
\newcommand{\WAS}{\mathbf{wAF}}
\newcommand{\CAF}{\mathbf{cAF}}
\newcommand{\A}{\mathcal{A}}
\newcommand{\C}{\mathcal{C}}
\newcommand{\Att}{Att}

\newtheorem{example}{Example}
\newtheorem{theorem}{Theorem}

\newtheorem{corollary}[theorem]{Corollary}
\newtheorem{proposition}[theorem]{Proposition}

\newtheorem{definition}{Definition}

\title{Eliciting Rational Initial Weights in Gradual Argumentation
}

\author{
  Nir Oren\\
  University of Aberdeen \\
  Scotland \\
  \texttt{n.oren@abdn.ac.uk} \\
   \And
  Bruno Yun \\
  Universite Claude Bernard Lyon 1,\\
  CNRS, Ecole Centrale de Lyon,\\
  INSA Lyon, Université Lumière Lyon 2, LIRIS, UMR5205\\
  69622 Villeurbanne\\
  \texttt{bruno.yun@univ-lyon1.fr} \\
}

\begin{document}
\maketitle

\begin{abstract}
Many semantics for weighted argumentation frameworks assume that each argument is associated with an initial weight.  However, eliciting these initial weights poses challenges: (1) accurately providing a specific numerical value is often difficult, and (2) individuals frequently confuse initial weights with acceptability degrees in the presence of other arguments.
To address these issues, we propose an elicitation pipeline that allows one to specify acceptability degree intervals for each argument. By employing gradual semantics, we can refine these intervals when they are rational, restore rationality when they are not, and ultimately identify possible initial weights for each argument.
\end{abstract}

\section{Introduction}


After the seminal work of \cite{dung_acceptability_1995}, the argumentation community's focus has shifted to the relation between arguments, abstracting the content of the arguments. Thus, debates or discussions are represented as directed graphs, with arguments as nodes and attacks between arguments as directed edges. 
Several \textit{extension-based semantics} have been defined to obtain conclusions from argumentation graphs. Such semantics identify  subsets of arguments (called extensions), representing consistent conclusions   \cite{baroni_introduction_2011,gaggl_cf2_2013,dvorak_comparing_2012}.

Motivated by the work of \cite{amgoud_ranking-based_2013}, researchers started to focus on semantics which could give a more gradual view on arguments' acceptability by ranking them from the ``\textit{less attacked}'' to the most (not necessarily based on the cardinality or quality of the attackers) \cite{bonzon_combining_2018}. A subset of these semantics, called \textit{gradual semantics}, assign to each argument, a numeric degree representing its strength after taking into account the relation between arguments (e.g., \cite{besnard_logic-based_2001}). More recent works expand gradual semantics with supports \cite{rago_discontinuity-free_2016}, sets of attacking arguments \cite{DBLP:conf/aaai/YunVC20}, and initial weights \cite{amgoud_weighted_2018}.
\textit{But, where do these initial weights come from?}

When asking people to elicitate the initial weights of arguments, people often struggle to identify an exact value. Complicating this further, in a setting where multiple arguments are presented, people may struggle to differentiate between the initial weight of an argument, and its final acceptability degree (which takes into account the other arguments). We therefore argue that --- from a knowledge engineering point of view --- it may make sense to elicit a single interval from a user which in some sense captures both the argument's initial weight and final acceptability degree. 

In this paper, we investigate how to narrow down an argumentation system's final acceptability degrees according to some desirable properties by using the elicited value as well as the structure of the argumentation framework. If no such rational intervals exist, we seek to find a rational value which requires minimal change from the elicited values. We also provide an implementation and evaluation of our approach\footnote{A user interface demonstrating the system is available at \url{https://tinyurl.com/3mb228fr}. 
Source code is in supplementary material.}.

In Section \ref{sec:background}, we recall the background on formal argumentation and some well-known gradual semantics. Section \ref{sec:CAF} sets up the theoretical framework for constrained argumentation frameworks, where arguments are associated with intervals, and defines the notion of rationality for them. In Section \ref{sec:refine-intervals}, we explain how constrained argumentation frameworks can be refined for a gradual semantics and we explore different strategies for correcting irrationality in Section \ref{sec-correct-irrationality}. Section \ref{sec:eval} evaluates the algorithms presented in the paper. Sections \ref{sec:related_worl} and \ref{sec:conclusion} discuss related  and future work.

\section{Background}
\label{sec:background}

We first introduce the necessary argumentation concepts.

A weighted argumentation framework is composed of arguments, attacks between arguments, and a weighting function associating each argument to an initial weight representing its trustworthiness or certainty degree of its premises.

\begin{definition}[Weighted Argumentation Framework]
An weighted argumentation framework (WAF) is $\WAS = (\A,\C, w )$, where $\A$ is a finite set of arguments, $\C \subseteq \A \times \A$ is a set of attacks between arguments, and $w$ is a weighting function from $\A$ to $[0,1]$. 
\end{definition}

Given $\WAS = (\A,\C, w)$ and $a \in \A$, $\Att(a) = \{ b \in \A \mid (b,a) \in \C)$ and $\Att^*(a) = \{ b \in Att(a) \mid w(a) \neq 0 \}$.
%
Let $x,y \in \A$. A \textit{path} from $y$ to $x$ is a sequence $\langle x_{0},\dots,x_{n} \rangle$ of arguments such that $x_{0} = y$, $x_{n} = x$ and $\forall i$ s.t. $0 \le i < n, (x_{i},x_{i+1}) \in \C$.

The usual Dung's extension-based semantics~\cite{dung_acceptability_1995} induces a two-levels acceptability of arguments (inside or outside of one or all extensions). Gradual semantics (and ranking-based semantics in general) have been proposed as a more fine-grained approach to argument acceptability~\cite{DBLP:conf/sum/AmgoudB13a,BonzonDKM23} and their weighted versions have been defined (e.g., \cite{amgoud_evaluation_2022}).
These gradual semantics use a weighting to assign to each argument in the weighted argumentation framework a score, called (acceptability) degree.

\begin{definition}[Weighted Gradual semantics] \label{def:gradualSemantics}
    A weighted gradual semantics is a function $\sigma$ which associates to each $\WAS = (\A,\C,w)$, a weighting $\sigma_\WAS: \A \to [0,1]$ on $\A$. $\sigma_\WAS(a)$ is the degree of $a$. 
\end{definition}

Let us now recall the well-known gradual semantics studied in this paper~\cite{besnard_logic-based_2001,DBLP:conf/ijcai/OrenYVB22,amgoud_evaluation_2022}.
The weighted h-categorizer semantics ($\nameHbs$) assigns a value to each argument by taking into account the initial weight of the argument, the degree of its attackers, which themselves take into account the degree of their attackers.

\begin{definition}[Weighted h-categorizer]\label{h-categorizer}
The weighted h-categorizer semantics is a gradual semantics $\Hbs$ s.t. for any $\WAS = (\A,\C,w)$ and $a \in \A$, $\Hbs_\WAS(a) = w(a)/(1+ \Sigma_{b \in \Att(a)} \Hbs_\WAS(b))$.
\end{definition}

The weighted Card-based semantics ($\nameCar$) favors the number of attackers over their quality. 
This semantics is based on a recursive function which assigns a score to each argument on the basis of its initial weight, the number of its direct attackers, and their degrees.

\begin{definition}[Weighted Card-based]\label{card-based}
    The weighted Card-based semantics is a gradual semantics $\Car$ s.t. for any $\WAS = (\A,\C,w)$ and $a \in \A$, $\Car_\WAS(a) = w(a)/(1+ |\Att^*(a)| + \frac{\Sigma_{b \in \Att^*(a)} \Car_\WAS(b)}{|\Att^*(a)|})$ if $\Att^*(a) \neq \emptyset$, otherwise $\Car_\WAS(a) = w(a)$.
\end{definition}

The weighted Max-based semantics ($\nameMax$) favors the quality of attackers over their number, i.e., the degree of an argument is based on its initial weight and the degree of its strongest direct attacker.

\begin{definition}[Weighted Max-based]\label{max-based}
    The weighted Max-based semantics is a gradual semantics $\Max$ s.t. any $\WAS = (\A,\C,w)$ and $a \in \A$, $\Max_\WAS(a) = w(a)/(1+ \max_{b \in \Att(a)} \Max_\WAS(b))$
\end{definition}

The usual setting starts from a weighted argumentation framework and utilizes a weighted gradual semantics to compute the acceptability degree. The approach of \cite{DBLP:conf/ijcai/OrenYVB22} has shown that we can instead start with the acceptability degree and the structure of the graph to infer the initial weights. \textit{But what if we are unsure about the acceptability degrees?} We explore this issue in the next section.

\section{Working with intervals}
\label{sec:CAF}

Humans often struggle to differentiate between an argument's inital weights and its final acceptability degree, as human reasoning implicitly considers additional background knowledge. Humans also struggle to pinpoint a specific final acceptability degree value. We thus propose a framework where one can indicate an interval for acceptability degrees, and where the framework then supports the identification of appropriate final acceptability degrees. This is shown in Figure \ref{fig:co-elicitation}. 

%
%
%

To elicit argument strengths, indivudals would first provide intervals (step 1).
%
If the input provided is rational (w.r.t. some weighted gradual semantics), then we refine the interval of linked arguments using the corresponding gradual semantics (step 2a). However, in the case where the appraisal is irrational (for a given gradual semantics), the user will have to re-do their appraisal with support from the framework (step 2b). Note that after the automatic refinement of the intervals, the user can continue to improve their input. Once users are satisfied with their acceptability degree intervals, they can then sample some possible valid initial weights.

\begin{figure}
    \centering
    \includegraphics[width=7.3cm]{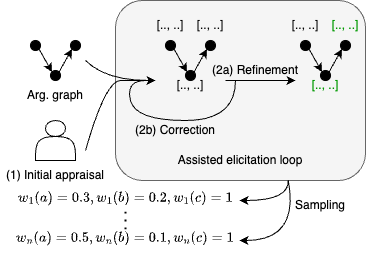}
    \caption{Assisted elicitation of argument strength.}
    \label{fig:co-elicitation}
\end{figure}

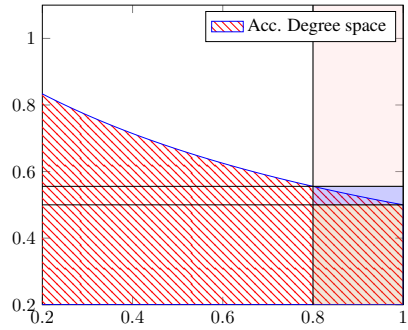
\begin{figure}
    \centering
    
\begin{tikzpicture}[scale=0.7,
        hatch distance/.store in=\hatchdistance,
        hatch distance=10pt,
        hatch thickness/.store in=\hatchthickness,
        hatch thickness=2pt
    ]
    \makeatletter
    \pgfdeclarepatternformonly[\hatchdistance,\hatchthickness]{flexible hatch}
    {\pgfqpoint{0pt}{0pt}}
    {\pgfqpoint{\hatchdistance}{\hatchdistance}}
    {\pgfpoint{\hatchdistance-1pt}{\hatchdistance-1pt}}%
    {
        \pgfsetcolor{\tikz@pattern@color}
        \pgfsetlinewidth{\hatchthickness}
        \pgfpathmoveto{\pgfqpoint{0pt}{0pt}}
        \pgfpathlineto{\pgfqpoint{\hatchdistance}{\hatchdistance}}
        \pgfusepath{stroke}
    }
    
\begin{axis}[xmin=0.2, ymin=0.2,xmax=1,ymax=1.1, samples=50]

    \addplot+[mark=none,
        domain=0:1,
        samples=100,
        pattern=north west lines,
        area legend,
        pattern color=red]{1/(1+x)} \closedcycle;

\addlegendentry{Acc. Degree space}

   \addplot[fill=brown, opacity=0.2] table {
      0.8 0 
      1 0 
      1 0.5
      0.8 0.5
    } -- cycle;

    \addplot[fill=blue, opacity=0.2] table {
      0.8 0.5 
      1 0.5 
      1 0.5556
      0.8 0.5556
    } -- cycle;

    \addplot[fill=pink, opacity=0.2] table {
      0.8 0.5556
      1 0.5556
      1 1.1
      0.8 1.1
    } -- cycle;
    
  \addplot +[black, solid, mark=none] coordinates {(0.8, 0) (0.8, 1.1)};
  \addplot +[black, solid, mark=none] coordinates {(1, 0) (1, 1.1)};
  \addplot +[black, solid, mark=none] coordinates {(0, 0.5) (1, 0.5)};
  \addplot +[black, solid, mark=none] coordinates {(0, 5/9) (1, 5/9)};
\end{axis}
\end{tikzpicture}

\caption{The area under the green curve represents the acceptability degree space for $a$ (x-axis) and $b$ (y-axis). The vertical lines are $x_1= 0.8$ and $x_2= 1$. The horizontal lines are at $y_1 = 0.5$ and $y_2= 5/9 \approx 0.555$.}
    \label{fig:ex1}
\end{figure}

\begin{definition}[Constrained Argumentation Framework]
A constrained argumentation framework (CAF) is $\CAF = (\A,\C,I)$, where $\AS = (\A,\C)$ is an (non-weighted) argumentation framework and $I$ is a function that associates an interval to each argument $a \in \A$  such that $I(a) \subseteq [0,1]$. 
\end{definition}

In our setting, an interval $I(a) = [0,1]$ for some argument $a$ effectively means that there is no restriction on $a$.

\begin{definition}
Given $\AS = (\A,\C)$ and semantics $\sigma$, where $\A = \{a_1, \dots, a_n\}$ and a point $P=(p_1,\ldots,p_n) \in [0,1]^n$ . We say that $P$ is in the \emph{acceptability degree space} w.r.t. $\sigma$ iff there is a weighting $w$ such that for every $a_i \in \A, \sigma_{(\A,\C,w)}(a_i)=p_i$.

We call the acceptability degree space (of $\AS$ w.r.t. $\sigma$) the set of all such points \cite{DBLP:conf/ijcai/OrenYVB22}.
\end{definition}

Let $\sigma$ be a gradual semantics and $\CAF =( \{a_1, a_2, \dots, a_n\},\C,I)$ be a constrained weighted argumentation framework.
We say that the pair $(\CAF, \sigma)$ is \emph{irrational} if there is no weighting $w: \A \to [0,1]$ such that for all $a_i \in \A$, $\sigma_{(\A,\C,w)}(a_i) \in I(a_i)$, and is \emph{rational} otherwise. 
We say that $(\CAF, \sigma)$ is \textit{fully rational} if for all $(d_1, d_2, \dots, d_n)$, where $d_i \in I(a_i)$ for all $ 1 \leq i \leq n$, there exists $w: \A \to [0,1]$ such that $\sigma_{(\A,\C,w)(a_i)} = d_i$.

\begin{example}
    Consider a CAF $(\{a,b\},\{(a,b)\},I)$. In Figure 
    \ref{fig:ex1}, we represent the acceptability degree space of $a$ and $b$ (below the blue curve) w.r.t. the weighted h-categorizer. For every point $(d_a, d_b)$ in this area, there exists a weighting function $w$ such that $\sigma_{(\A,\C,w)}(a) = d_a$ and $\sigma_{(\A,\C,w)}(b) = d_b$.
        
    Assume that $I(a)=[0.8,1.0]$. Let $U=5/9$ and $L=0.5$.
    \begin{itemize}
      \item If $I(b) \in [x,L]$, for $0 \leq x \leq L$, then the CAF is fully rational as for every $(d_a, d_b), d_a \in I(a),$ and $d_b \in I(b)$, there exists a weighting function $w$ s.t.\  $\sigma_{(\A,\C,w)}(a) = d_a$ and $\sigma_{(\A,\C,w)}(b) = d_b$.  Any such point $(d_a,d_b)$ will lie within the brown rectangle shown in Figure \ref{fig:ex1}. Thus, any point in $I(a) \times I(b)$ lies within the rectangle.
    
      \item If $I(b)=[x,y]$, for $L < y, 0 \leq x\leq U$ and $x \leq y$, then the CAF is rational, as the point $(0.8, x) \in I(a) \times I(b)$ is in the brown or blue area but it is not fully rational as the point $(1,y) \in I(a) \times I(b)$ is not in the acceptability degree space (but still lies inside the blue area).
      
     \item Finally, if $I(b)=[x,y]$ such that $U < x \leq y$, then the CAF is irrational because every point in $I(a) \times I(b)$ lies in the pink area (and thus outside of the green area).
    \end{itemize}

Note that intervals on acceptability degree cannot be directly converted into intervals for initial weights. This is due to the interdependencies between arguments in the final acceptability degree computation process. For example, if $I(a) = [0.8, 1]$ and $I(b) = [0,0.5]$, one would think that the initial weights of $a$ lies between $0.8$ and $1$ whereas the one of $b$ would lie within $0$ and $1$. Notice that if $w(a)=0.8$ and $w(b)=1$, we would have $\Hbs_\WAS(b) \notin I(b).$
\end{example}

If $(\CAF, \sigma)$ is fully rational, then there is a weighting for any final acceptability degrees within the intervals and is thus \emph{achievable}. In contrast, if $(\CAF, \sigma)$ is rational but not fully rational, then some combination of final acceptability degrees within the interval will not be achievable. $\epsilon$-rationality serves as a measure for ``\textit{how close to fully rational}'' a $(\CAF, \sigma)$ pair is.

\begin{definition} 
Given $\CAF= (\A,\C,I)$, a gradual semantics $\sigma$ and $0 \leq \epsilon $, we say $(\CAF, \sigma)$  is $\epsilon$-rational if for all $a \in \A$, it holds that:
\begin{itemize}
    \item $ \epsilon \leq \max(I(a)) - \min(I(a))$; and
    \item $((\A,\C,I_a) , \sigma)$ is rational, where for all $a' \in \A, a'\neq a$, $I_a(a')=I(a')$ and $I_a(a)=[max(I(a))-\epsilon,max(I(a))]$, 
\end{itemize}
\end{definition}

According to the first item of this definition, $\epsilon$ must be smaller or equal to the size of the smallest interval in the CAF. 
The second item of the definition considers a set of induced CAFs (one per argument) where for each CAF, a specific argument's interval is modified so that its lower bound is equal to its upper bound minus $\epsilon$ (and the remaining arguments are unchanged). If all such induced CAFs are rational, then --- in combination with the first condition --- the original CAF is $\epsilon$-rational.
Intuitively, this means that the shape $\times_{a \in \A} I(a)$ is not ``too far'' (i.e., within $\epsilon$) from the acceptability degree space on all axes.

Observe that a fully rational CAF must be $0$-rational, any $\epsilon$-rational (where $\epsilon >0$) CAF is rational (but not fully rational), while an irrational CAF is not $\epsilon$-rational for any $\epsilon$. Note that these are not "if and only if" relationships. Thus, a $0$-rational CAF is not necessarily fully rational, but is always rational.


\begin{proposition}
If ($\CAF, \sigma$) is $\epsilon$-rational, then ($\CAF, \sigma$) is rational.
\end{proposition}

\begin{proof}
Let $\CAF = (\A,\C,I)$ and assume ($\CAF, \sigma$) is $\epsilon$-rational. For any arbitrary $a\in \A$, and $0 \leq \epsilon \leq \max(I(a))-\min(I(a))$, we have that $((\A,\C,I_a) , \sigma)$ is rational. 
This means that there exists $w : \A \to [0,1]$ such that $\sigma_{(\A,\C,w)}(a)  \in [ \max(I(a)) - \epsilon, \max(I(a))]$ and for all $b \in \A \setminus\{a\}, \sigma_{(\A,\C,w)}(b) \in I(b).$ 
Notice that  $ \min(I(a)) \leq \max(I(a)) - \epsilon \leq \sigma_{(\A,\C,w)}(a) \leq \max(I(a))$, thus $\sigma_{(\A,\C,w)}(a) \in I(a)$. Thus, ($\CAF, \sigma$) is rational.
\end{proof}

Rationality captures whether some part of the interval is in the acceptability degree space. A stronger concept is one of \emph{refinement}. When refining intervals, we seek to determine whether the intervals can be tightened without losing any part of the acceptability degree space.

\begin{definition}
Let $(\CAF = (\A, \C, I) , \sigma)$ be a rational pair of constrained argumentation framework and a semantics, we say that $\CAF' = (\A, \C, I')$ is a \emph{refinement} of $\CAF$ (w.r.t. $\sigma$) iff (1) for all $a \in \A$, $I'(a) \subseteq I(a)$ and (2) for all $w: \A \to [0,1]$ such that for all $a\in \A$,  $\sigma_{(\A,\C,w)}(a) \in I(a)$, then $\sigma_{(\A,\C,w)}(a) \in I'(a)$.
\label{def:refinement}
\end{definition}

Condition (2) is important and specifies that a refinement only removes unachievable acceptability degrees. A refinement thus shrinks intervals while ensuring that any achievable acceptability degrees remain. This means that a refinement can only reduce the upper bound of an interval.

\begin{proposition}
If ($\CAF, \sigma$) is rational and $\CAF'$ is a refinement of it, then ($\CAF', \sigma$) is rational.

\end{proposition}

\begin{proof}
Assume that ($\CAF, \sigma$) is rational, then there is a weighting $w : \A \to [0,1]$ such that for all $a \in \A, \sigma_{(\A,\C,w)}(a) \in I(a).$ From the second item of Definition \ref{def:refinement}, we get that for all $a \in \A, \sigma_{(\A,\C,w)}(a) \in I'(a).$ This means that $(\CAF', \sigma)$ is rational.
\end{proof}

Next we show an important intermediate result which informs us on the shape of the ``acceptability degree space'' of argumentation graphs. 
We introduce the notion of \emph{axial-radiality} for a semantics. Intuitively an axial-radial semantics remains within the acceptability degree space if moving from a point within the space  towards the origin.

\begin{definition}
A weighted gradual semantics $\sigma$ is said to be obey \emph{axial-radiality} iff for any argumentation framework $\WAS= (\A, \C,w)$, corresponding $\sigma_\AS$ and $a\in \A$ then for all $\varepsilon \in [0, \sigma_\AS(a)]$, $((\A,\C, I), \sigma)$ is rational, where $I(a)= [\sigma_\AS(a) - \varepsilon, \sigma_\AS(a) - \varepsilon]$ and for all $b \in \A \setminus \{a\}, I(b) = [\sigma_\AS(b), \sigma_\AS(b)].$
\end{definition}

For semantics satisfying axial-radiality, the corner points of the hyper-rectangle induced by the intervals over arguments can be used to determine whether the pair CAF/semantics is or is not (fully) rational.

\begin{proposition}
Given $\CAF = (\A,\C, I)$, $\A = \{ a_1, \dots, a_n \}$, a semantics $\sigma$ satisfying axial-radiality, $Z_i = \{\min(I(a_i)), \max(I(a_i)) \}$ and $Z = Z_1 \times Z_2 \times \dots \times Z_n \subseteq [0,1]^n $ be the set of all corner points of the hyper-rectangle induced by $I$. It holds that:
\begin{itemize}
    \item If all points of $Z$ are in the acceptability degree space (w.r.t. $\sigma$), then $(\CAF,\sigma)$ is fully rational.
    \item If none of the points of $Z$ are in the acceptability degree space (w.r.t. $\sigma$), then $(\CAF,\sigma)$ is irrational.
    \item If $1 \leq m\leq 2^n-1$ points of $Z$ are not in the acceptability degree space (w.r.t. $\sigma$), then $(\CAF,\sigma)$ is rational but not fully rational.
\end{itemize}
\label{prop:points_interval}
\end{proposition}

\begin{proof}
    The point $(\max(I(a_1)), \dots, \max(I(a_n))) \in Z$ is in the acceptability degree space and from the definition of axial-radiality, we conclude that every point $(d_1, d_2, \dots, d_n)$, where $d_i \in I(a_i)$, is in the acceptability degree space, thus, $(\CAF,\sigma)$ is fully rational.
    
    Assume that non of the points of $Z$ are in the acceptability degree space but that $(\CAF,\sigma)$ is rational. This means that there exists $P = (d_1, d_2, \dots, d_n)$, where $d_i \in I(a_i)$, such that $P$ is in the acceptability degree space. Using axial-radiality, we conclude that $(\min(a_1), \dots, \min(a_n))$ is also in the acceptability degree space, contradiction.
    
    Since one, but not all points, inside the acceptability space,  $(\CAF, \sigma)$ is rational but not fully rational.
\end{proof}

Item 3 of Proposition \ref{prop:points_interval} does not mean that there exists a non-trivial refinement. Intuitively, a refinement exists if (at least) an entire face of a $\CAF$ lies outside the acceptability degree space and at least one corner point lies within it. This is formalised as follows.

\begin{proposition}
Let $\CAF = (\A,\C, I)$, $\A = \{ a_1, \dots, a_n \}$, a semantics $\sigma$ satisfying axial-radiality, and $Z$ be the set of all corner points of the hyper-rectangle induced by $I$ such that $(\CAF, \sigma)$ is rational. There is a non trivial refinement of $\CAF$ (w.r.t. $\sigma$) iff there exists a set $X_i \subseteq Z$ composed of $2^{n-1}$ corner points such that every $(x_1, \dots, x_n) \in X_i$ is not in the acceptability degree space and $x_i = \max(I(a_i))$.
\label{cond:non-trivial-refinement}
\end{proposition}

\begin{proof}
The sketch of the proof is as follows.
    $\Leftarrow$
Let us assume that $X_i$ exists such that none of the points in $X_i$ are in the acceptability degree space. Let $P$ be the bounded hyperplane defined by the points in $X_i$. We know that none of the points in $P$ are in the acceptability degree space because of axial-radiality (otherwise, one of the points in $X_i$ would be in the acceptability degree space). There exists $\varepsilon>0$ such that $X'_i = \{ (x_1, \dots, x_{i-1}, x_i - \varepsilon, x_{i+1}, \dots, x_n ) \mid (x_1, \dots, x_n) \in X_i\}$ and all points in $X'_i$ are also not in the acceptability degree space. It is easy to see that $\CAF' = (\A,\C, I')$, where $I'(a) = [\min(a), x_i - \varepsilon]$ if $a = a_i$ and $I'(a) = I(a)$ otherwise, is a non trivial refinement of $\CAF$.

    $\Rightarrow$
    Let assume that $\CAF = (\A,\C,I)$ has a non trivial refinement $\CAF' = (\A,\C,I')$ such that for all $a \in \A, I'(a) \subseteq I(a).$ Let $a_i$ be such that $I'(a_i) \subset I(a_i)$ (this exists as the refinement is non trivial). 
    Let $X_i =\{ (x_1, x_2, \dots, x_{i-1}, \max(I(a_i)), \dots, x_n) \mid $ where for any $j \neq i$, $ x_j \in \{ \max(I(a_j)), \min(I(a_j))\}\}$ be the face of the hyper-rectangle in direction $i$, none of the points of $X_i$ can be in the acceptability degree space, otherwise $\CAF' = (\A,\C,I')$ would not be a refinement.  
\end{proof}








%
Given a rational $(\CAF = (\A, \C, I) , \sigma)$ and two of its refinements $\CAF_1$ and $\CAF_2$, we say that $\CAF_1$ is a \textit{better refinement} than $\CAF_2$ iff $\CAF_1$ is a refinement of $\CAF_2$.

\subsection{Determining the rationality of CAFs}

In this section, we focus on several popular weighted gradual semantics defined in the literature identify properties of CAFs with regards to these semantics.

\begin{proposition}
Let $\CAF$ be an arbitrary CAF, if the pair $(\CAF , \Hbs)$ is a rational then $(\CAF , \Max)$ is rational.
\label{prop:hbs-to-max}
\end{proposition}

\begin{proof}
Let $\CAF = (\A,\C, I)$ such that $(\CAF , \Hbs)$ is rational. This means that there exists $w : \A \to [0,1]$ such that for all $a\in \A$, $\Hbs_{(\A,\C,w)}(a) \in I(a)$. We show that we can craft a weighting $w' : \A \to [0,1]$ such that for all $\Max_{(\A,\C,w')}(a) = \Hbs_{(\A,\C,w)}(a)$ (and thus $\Max_{(\A,\C,w')}(a) \in I(a)$).
\end{proof}

Using Prop.\  5.6 of \cite{DBLP:journals/corr/abs-2203-01201}, we know that such weighting $w'$ exists and can be calculted efficiently using  matrix operations, as per the following example.

\begin{example}
\label{ex-hbs-to-max}
Let us consider the $\CAF$ displayed in Figure \ref{fig:prop1-ex}. The pair $(\CAF, \Hbs)$ is rational as the weighting $w$ such that for all $a \in \A, w(a)=1$ yields $\Hbs_{(\A,\C,w)}(a) = \Hbs_{(\A,\C,w)}(b) = 0.618$, $\Hbs_{(\A,\C,w)}(c) = 0.447$, and $\Hbs_{(\A,\C,w)}(d) = 0.691$.
A weighting $w'$, such that for all $a \in \A, \Hbs_{(\A,\C,w)}(a) = \Max_{(\A,\C,w')}(a)$, is defined via the vector $W'$, obtained using the following matrix operations:
$$W' = 
	\begin{bmatrix} 
	0.618 \\
	0.618\\
	0.447 \\
        0.691 \\
	\end{bmatrix} + 
 \begin{bmatrix} 
	0.618 & 0 & 0 & 0  \\
	0 & 0.618 & 0 & 0 \\
	0 & 0 & 0.447 & 0 \\
 0 & 0 & 0 & 0.691  \\
	\end{bmatrix}
 *\max(X)
	$$
where $X$ is the element-wise multiplication of two matrices and $\max(X)$ takes the largest element of each row of $X$.
$$X = 
	\begin{bmatrix} 
	0 & 1 & 0 &0 \\
	1&0&0&0\\
    1&1&0&0\\
    0&0&1&0\\
	\end{bmatrix} \odot
\begin{bmatrix} 
0.618 & 0.618 & 0.447 & 0.691 \\
0.618 & 0.618 & 0.447 & 0.691 \\
0.618 & 0.618 & 0.447 & 0.691 \\
0.618 & 0.618 & 0.447 & 0.691 \\
\end{bmatrix}
	$$

We obtain that $W'= (1,1, 0.723,1)^T$, meaning that $w'(a) = w'(b) = w'(d) = 1$ and $w'(c) = 0.723$.
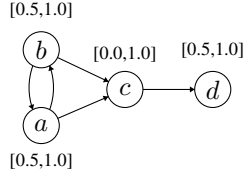
\begin{figure}
    \centering
    \begin{tikzpicture}[scale=0.08]
\tikzstyle{every node}+=[inner sep=0pt]
\draw [black] (26.8,-23.1) circle (3);
\draw (26.8,-23.1) node [label={[label distance=8] 90: \scriptsize  [0.5,1.0]}] {$b$};
\draw [black] (26.8,-35.7) circle (3);
\draw (26.8,-35.7) node[label={[label distance=8] 270: \scriptsize  [0.5,1.0]}] {$a$};
\draw [black] (40.7,-29.7) circle (3);
\draw (40.7,-29.7) node [label={[label distance=8] 90: \scriptsize  [0.0,1.0]}]{$c$};
\draw [black] (55.3,-29.7) circle (3);
\draw (55.3,-29.7) node [label={[label distance=8] 90: \scriptsize  [0.5,1.0]}]{$d$};
\draw [black] (25.408,-33.054) arc (-160.2119:-199.7881:10.792);
\fill [black] (25.41,-33.05) -- (25.61,-32.13) -- (24.67,-32.47);
\draw [black] (28.197,-25.744) arc (19.86919:-19.86919:10.758);
\fill [black] (28.2,-25.74) -- (28,-26.67) -- (28.94,-26.33);
\draw [black] (29.51,-24.39) -- (37.99,-28.41);
\fill [black] (37.99,-28.41) -- (37.48,-27.62) -- (37.05,-28.52);
\draw [black] (29.55,-34.51) -- (37.95,-30.89);
\fill [black] (37.95,-30.89) -- (37.01,-30.75) -- (37.41,-31.67);
\draw [black] (43.7,-29.7) -- (52.3,-29.7);
\fill [black] (52.3,-29.7) -- (51.5,-29.2) -- (51.5,-30.2);
\end{tikzpicture}

    \caption{Representation of a contrained argumentation framework.}
    \label{fig:prop1-ex}
\end{figure}
\end{example}

Note that the inverse of Proposition \ref{prop:hbs-to-max} is not true.

\begin{proposition}
\label{prop:shape_border}
Let $\sigma \in \{ \Hbs, \Max\}$, it holds that $\sigma$ satisfies axial-radiality.

\begin{proof}
From \cite{DBLP:journals/corr/abs-2203-01201}, we know that -- given a set of arguments and attacks between them -- there is a one-to-one equivalence between a vector of weights ($\overrightarrow{w}$) and a vector of acceptability degrees ($\overrightarrow{X}$). For the weighted h-categorizer, one can go from the latter to the former using the equation $\overrightarrow{w} = \overrightarrow{X}+\mathbb{M}\mathbb{A}\overrightarrow{X}$, where $\mathbb{M}$ is the diagonal matrix with the acceptability degrees and $\mathbb{A}$ is the adjacency matrix.
From this equation, one can check whether a vector of acceptability degrees is achievable or not by computing the vector of weights and checking if its values are between $0$ and $1$. 

Consider an  argumentation framework $\AS = (\A,\C,w)$ and corresponding $\Hbs$, vectors $\overrightarrow{w}$,  $\overrightarrow{X}$, and matrices $\mathbb{M}$, $\mathbb{A}$. 

Now, take any $a\in \A$ and $\varepsilon \in [0, \sigma_\AS(a)]$ and compute the vector $\overrightarrow{X'}$ where the value at the index corresponding to $a$ is updated to $\sigma_\AS(a) - \varepsilon$. Similarly, we obtain $\mathbb{M'}$. 
We can observe that $\overrightarrow{w'} = \overrightarrow{X'}+\mathbb{M'}\mathbb{A}\overrightarrow{X'}$ has positive values as $\overrightarrow{X'}$, $\mathbb{M'}$ and $\mathbb{A}$ only have positive values. Moreover, we have that $\overrightarrow{w'} \leq \overrightarrow{w}$, meaning that $\overrightarrow{w'}$ has values inferior or equal to $1$. We conclude that $((\A,\C, I), \Hbs)$ is rational. Using Proposition \ref{prop:hbs-to-max}, we conclude the same for $\Max$.
\end{proof}

\label{prop:const-space}
\end{proposition}

The next proposition  answers the question \textit{how can we determine if a pair $(\CAF, \sigma)$ is rational/fully rational?}

\begin{proposition}
It holds that $((\A, \C, I) , \sigma)$ is rational iff $((\A, \C, I') , \sigma)$ is rational, where $\sigma \in \{\Hbs, \Max \}$ and for all $a \in \A$, $I'(a) = [\min(I(a)), \min(I(a))]$.
\label{prop:rational}
\end{proposition}

\begin{proof}%
    $\Leftarrow$
    Assume that $((\A, \C, I') , \sigma)$ is rational, then it is obvious that $((\A, \C, I) , \sigma)$ is also rational (by definition). 

    \item[$\Rightarrow$] We show that if $((\A, \C, I) , \sigma)$ is rational then $((\A, \C, I') , \sigma)$ is rational.
%
Assume that $\A = \{a_1, a_2, \dots, a_n\}$ and $w$ such that $\AS= (\A,\C,w)$ and $\sigma_\AS(a) \in I(a)$ for all $a \in \A$.
By applying Proposition \ref{prop:shape_border}, we obtain that $((\A,\C,I_1), \sigma)$ is rational, where $I_1(a_1) = [\min(I(a_1)), \min(I(a_1))]$ and for all $b \in \A \setminus \{a_1\}$, $I_1(b) = I(b)$. Repeating this process $n-1$ more times proves that $((\A, \C, I') , \sigma)$ is rational.
\end{proof}


\begin{corollary}
Let $((\A,\C, I))$ be a CAF such that for all $a \in \A, \min(I(a))= 0$. Then, for $\sigma \in \{\Hbs, \Max\}$, $((\A,\C, I), \sigma)$ is rational.
\end{corollary}

Figure \ref{fig:rationality} illustrates the intuition of Proposition \ref{prop:rational}. Each axis represent the values of the degree of one argument (the figure represents a scenario with 3 arguments). The blue shape is the acceptability degree space and the red shape corresponds to the intervals for each argument given by $I$. $((\A,\C,I), \sigma)$ is rational iff there is an intersection between the red shape and the blue shape. 
This happens iff the ``minimal'' corner of the red shape is inside the blue shape.

\begin{proposition}
$((\A, \C, I) , \sigma)$ is fully rational iff $((\A, \C, I') , \sigma)$ is rational, where $\sigma \in \{\Hbs, \Max \}$ and for all $a \in \A$, $I'(a) = [\max(I(a)), \max(I(a))]$.
\label{prop:rational2}
\end{proposition}

\begin{proof}
The proof is similar to the one of Proposition \ref{prop:rational}.
\end{proof}

\begin{figure}
    \centering
    \includegraphics[width=8.0cm]{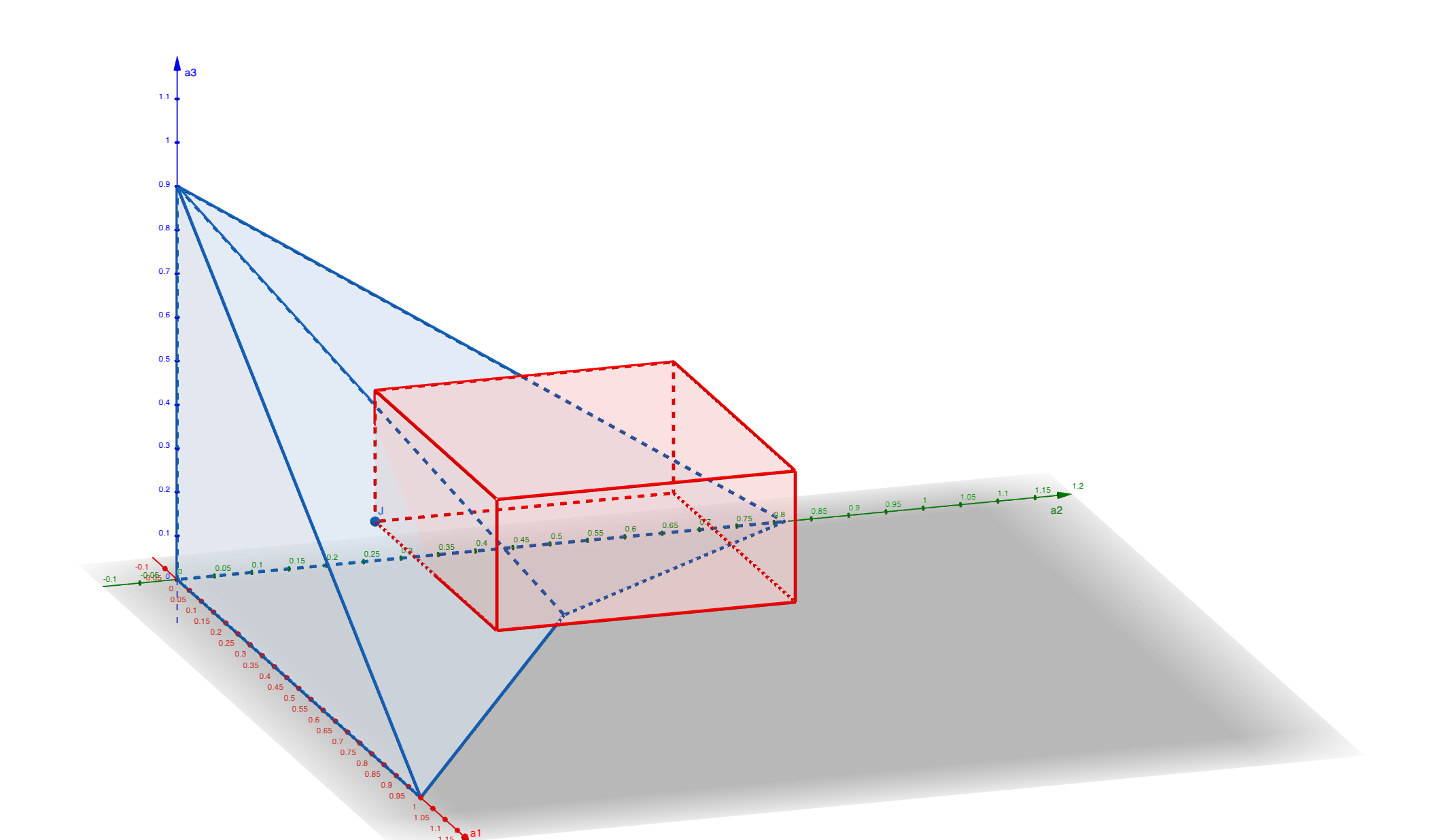}
    \caption{Representation of Proposition \ref{prop:rational}.}
    \label{fig:rationality}
\end{figure}


\subsection{Refining intervals}
\label{sec:refine-intervals}

We now consider the problem of obtaining the ``best'' refinement of a rational constrained argumentation frameworks w.r.t. some semantics. Such a best refinment is an $\epsilon$-rational refinement for which no elements of the acceptability degree space are lost, and which cannot be further refined without losing some elements of the acceptability degree space. Algorithm \ref{alg:refinement} describes how such a refinement can be identified.
Lines \ref{line:loop-alg1} to \ref{line:check_face} check that the condition of Proposition \ref{cond:non-trivial-refinement} is satisfied and that we can find a non-trivial refinement.
Lines \ref{line-start-modif-alg1} to \ref{line-end-modif-alg1} modify the argument intervals using the bisection method until we obtain an $\varepsilon$-refinement.
Note that the output is not fully-rational due to line \ref{output-cond-alg1}.
The use of the bisection method within the while loop means that the distance between $l$ and $r$ is halved at each iteration. Thus, the algorithm will converge in approximately $\log \varepsilon$ iterations.
Figure \ref{fig:best-refinement} illustrates the best refinement (orange) of a CAF (red).

\begin{algorithm}[t]
\begin{algorithmic}[1]

\Require A rational pair $((\A,\C,I), \sigma)$ and $\varepsilon>0$.
\For{$a\in \A$} \label{line:loop-alg1}

\State $\CAF_a \leftarrow ((\A,\C,I'), \sigma)$ where for all $b \in \A \setminus \{a\}, I'(b) = I(b)$ and $I'(a) = [\max(I(a)), \max(I(a))]$

\If{$ \CAF_a$ is irrational} \label{line:check_face}
    
    \State $l \leftarrow \min(I(a))$ \label{line-start-modif-alg1}, $r \leftarrow \max(I(a))$
    \State $\mathit{found} \leftarrow \mathit{False}$
    \While{$\neg \mathit{found}$}
        
        \State $m \leftarrow (l+r)/2$
        \State $\CAF'_a \leftarrow ((\A,\C,I''), \sigma)$ where for all $b \in \A \setminus \{a\}, I''(b) = I(b)$ and $I''(a) = [m,m]$
        \If{$\CAF'_a$ is irrational} \label{output-cond-alg1}
            \State $r \leftarrow m$
            \If{$(r-l)/2 < \varepsilon$}
                \State $I(a) \leftarrow [\min(I(a)), m]$
                \State $\mathit{found} \leftarrow \mathit{True}$
            \EndIf
        \Else
            \State $l \leftarrow m$
        \EndIf
        
    \EndWhile \label{line-end-modif-alg1}
    
\EndIf
\EndFor

\caption{Computing the best refinement}\label{alg:refinement} 
\label{alg1}

\end{algorithmic}

\end{algorithm}


\begin{figure}
    \centering
    \includegraphics[width=8.0cm]{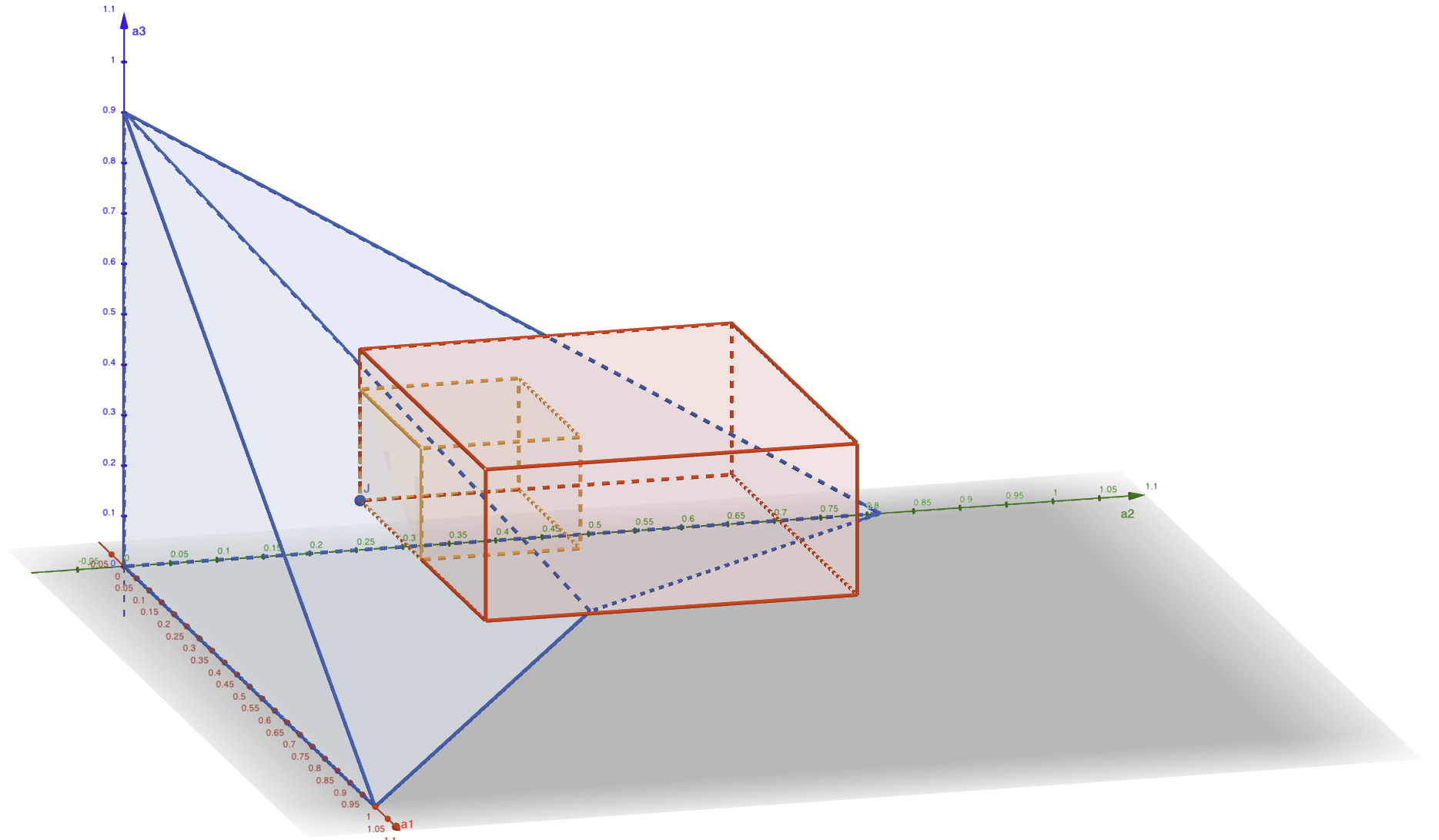}
    \caption{Representation of an $\varepsilon$-refinement.}
    \label{fig:best-refinement}
\end{figure}

\begin{proposition}
Given $(\CAF, \sigma$) s.t.~$\sigma$ is axial-radial, Alg.~\ref{alg1} outputs an $\varepsilon$-rational refinement $\CAF'$ of  $(\CAF, \sigma$).
\end{proposition}

\subsection{Correcting irrationality}
\label{sec-correct-irrationality}

In this section, we consider the situation where a constrained argumentation framework under a given weighted gradual semantics is irrational. In this situation, we must determine how the minimal elements of the interval must be changed to achieve rationality. We  consider the following strategies.

\paragraph{Strategy 1}
We consider a strategy where we associate a cost with modifying an argument, according to the cost function $cost: \A \to \mathbb{N}^*$. To change an argument $a$ by $\Delta$ therefore costs $cost(a) * \Delta$. We can thus consider a line intersecting the point $S = (\min(I(a_1)), \dots, \min(I(a_n)))$ of the acceptability degree space with gradient vector $1/cost(a_i)$ for component $i$. Moving along this line has equivalent cost for all arguments. Thus, our strategy consists of starting at point $S$ and creating a new CAF (with interval function $I'$) whose minimum interval moves along the negation of the gradient until this new CAF is rational, or until $\min (I'(a_i))$ is equal to 0 (for some $i$). In the former case, we return the new CAF, while in the latter, we continue moving down while keeping $\min(I'(a_i))=0$.
This approach is described in Algorithm \ref{alg:strat2}, which uses the bisection method to efficiently traverse the gradient line. Lines \ref{comp-alg}-\ref{end-comp-alg} compute the start point of the line (noting that the end point lies at 1). In lines \ref{begin-bisection}--\ref{end-bisection}, we use the bisection method to identify the extreme point where the resultant intervals are rational, taking the possibility of $\min(I'(a))$ being below 0 into account (line \ref{update-CAF}). 







\begin{algorithm}[t]
\begin{algorithmic}[1]

\Require An irrational $((\A,\C,I), \sigma)$, s.t.~$\sigma$ satisfies axial-radiality and $\varepsilon>0$.

  \State $l \gets 1$, $u \leftarrow 1$  \label{comp-alg}
  \ForAll{$a \in \A$}
    \State $t \gets 1 - (1/costs(a))*\min(I(a))$
    \If{$t \leq l_0$}
      \State $l \gets t$
    \EndIf
  \EndFor \label{end-comp-alg}

\While{$u- l \geq \varepsilon$} \label{begin-bisection}

    \State $m \leftarrow  (l+u)/2$
    \State $\CAF' \leftarrow ((\A,\C,I'), \sigma)$ s.t.~$\forall a \in \A$, $I'(a) = [\max(0,\frac{m}{costs(a)}+ min(I(a)) - costs(a)), \max(I(a))]$ \label{update-CAF} 
    
    \If{$\CAF'$ is rational}
        \State $l \leftarrow m$
    \Else
        \State $u \leftarrow m$
    \EndIf
\EndWhile \label{end-bisection}

\State \Return $((\A,\C,I''),\sigma)$ where for every $a \in \A, I''(a) = [\max(0,l/costs(a)+\min(I(a))),\max(I(a))]$ \label{return-CAF-rational}

\caption{Computing rational CAF using strategy 1}\label{alg:strat2}
\end{algorithmic}

\end{algorithm}

\paragraph{Strategy 2}
exhaustively applies Strategy 1 but restricts it so that only the intervals of a subset of arguments $X \subseteq \A$ can be modified.
When all argument subsets have been explored, the cost minimising subset is selected. 

\paragraph{Other strategies}
We can consider two special cases of the previous strategies that do not require arguments costs. First, we may consider the case where all arguments have equal cost, i.e.,~ $costs(a_i)=1$, for all $1 \leq i \leq n $. Second, we may wish to move along the line segment between the origin and  $S$, in which case our gradient vector has $\min(I(a_i))$ as its $i$-th component. 
This latter approach has the benefit of no argument having 0 as its minimum element in the interval (except if all of them have 0 as the minimum interval element).
We note that such a strategy can be captured using an appropriate cost function and therefore do not consider these further.


The aim of all strategies is to identify the closest point to $S$ in the acceptability degree space. Cost aware strategies change the metric by which this distance is measured. The need for strategies arises as --- in general --- one cannot identify the closest rational point to another point analytically.





    
        
        
    

\begin{figure*}[t]
\centering
\begin{subfigure}{4.4cm}
  \centering
  \includegraphics[width=0.9\textwidth]{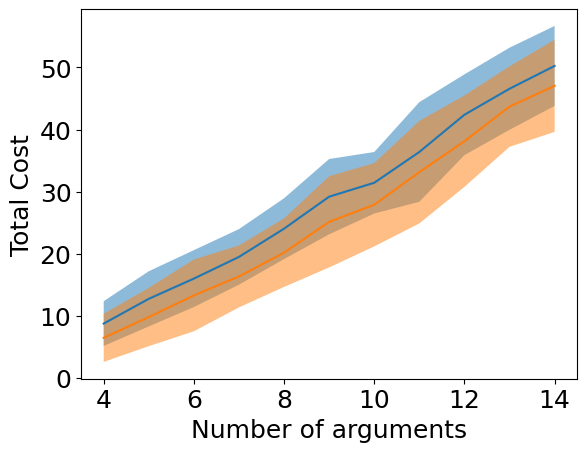}
  \caption{Weighted h-categorizer}
  \label{fig:h-cat-heuristic}
\end{subfigure}
\hspace{1cm}
\begin{subfigure}{4.4cm}
  \centering
  \includegraphics[width=0.9\textwidth]{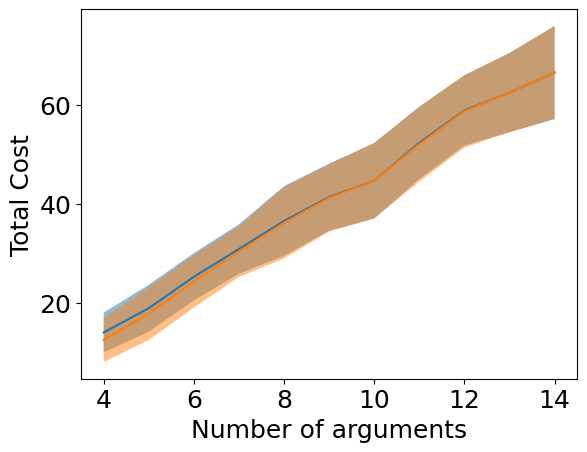}
  \caption{Weighted card-based}
  \label{fig:card-heuristic}
\end{subfigure}
\hspace{1cm}
\begin{subfigure}{4.4cm}
  \centering
  \includegraphics[width=0.9\textwidth]{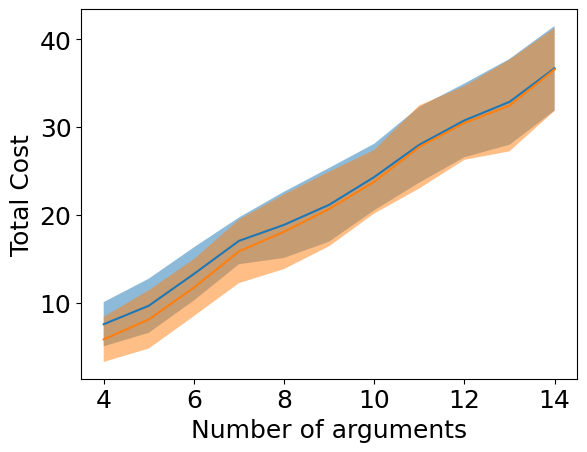}
  \caption{Weighted max-based}
  \label{fig:mbs-heuristic}
\end{subfigure}
\caption{\label{fig:eval1} Total costs obtained  applying Strategy 1 (blue) and Strategy 2 (orange). Shaded areas indicate one standard deviation.}
\end{figure*}

\begin{figure*}
\centering
\begin{subfigure}{4.4cm}
  \centering
  \includegraphics[width=0.9\textwidth]{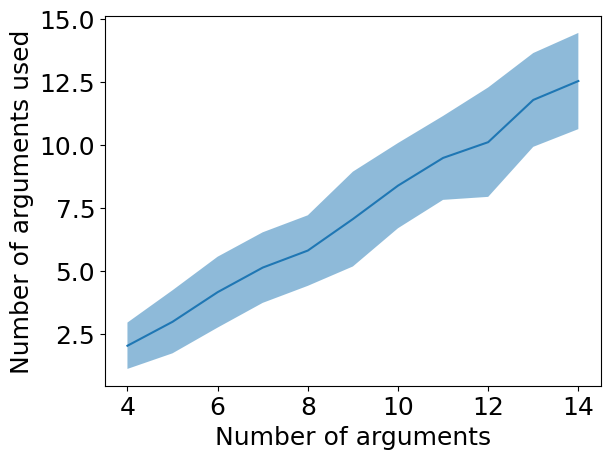}
  \caption{Weighted h-categorizer}
  \label{fig:h-cat-na}
\end{subfigure}
\hspace{1cm}
\begin{subfigure}{4.4cm}
  \centering
  \includegraphics[width=0.9\textwidth]{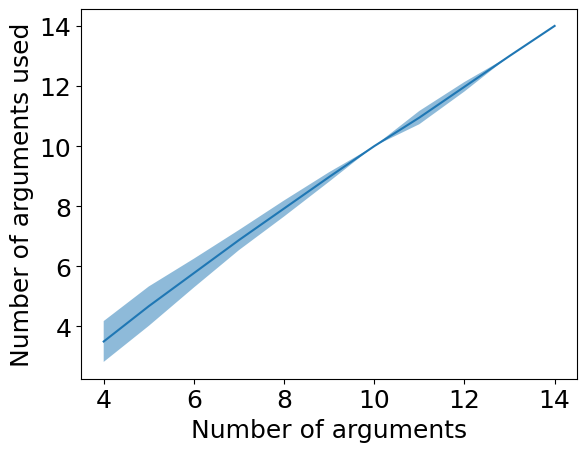}
  \caption{Weighted card-based}
  \label{fig:card-na}
\end{subfigure}
\hspace{1cm}
\begin{subfigure}{4.4cm}
  \centering
  \includegraphics[width=0.9\textwidth]{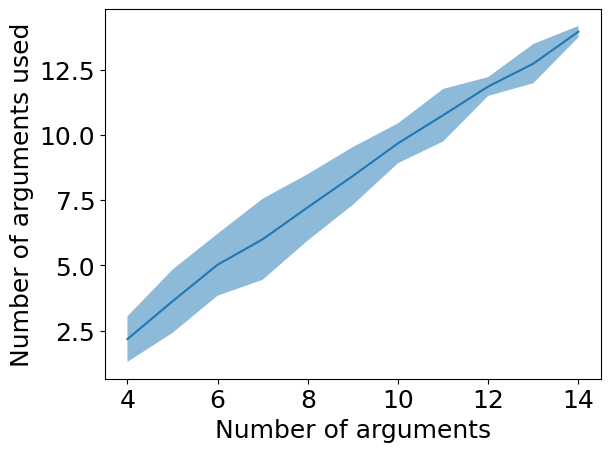}
  \caption{Weighted max-based}
  \label{fig:mbs-na}
\end{subfigure}

\caption{\label{fig:eval2} The total number of arguments used by Strategy 2 averaged across 40 runs, with the shaded area indicating one standard deviation.}
\end{figure*}

\section{Evaluation}
\label{sec:eval}
To evaluate the strategies from the previous section, 
%
we generated 40 Erd\"os-R\'enyi random graphs for each number of arguments between 4 and 14, with a 0.5 likelihood for edge creation. For each $(\A,\C)$, we initialised $I$ s.t., $I(a) = [x,1]$, where $x$ is picked randomly from $[0.8,1]$, 
ensuring that $((\A,\C,I), \sigma)$ is irrational according to the chosen semantics. 
Argument costs were drawn from the interval $[1,11]$.

Figure \ref{fig:eval1} illustrates the  cost for our cost-aware strategies using the weighted h-categorizer, card-based and max-based semantics. 
For the weighted h-categorizer semantics, Strategy 2 outperforms Strategy 1, while the advantage of the former strategy shrinks for the weighted card-based and max-based semantics as the number of arguments grows. Figure \ref{fig:eval2} plots the total number of arguments modified to achieve minimal cost via Strategy 2.
In all cases, the number of arguments grows linearly. However, we can see that the variance in the number of arguments modified remains large for the weighted h-categorizer, while shrinking for weighted max-based semantics. For the weighted card-based semantics, the variance begins small, and quickly approaches $0$ as the number of arguments increases. 
We hypothesise that the explicit consideration of the number of attacks in  weighted-card based semantics reduces the impact of arguments' initial weight on rationality, while the use of only maximally acceptable arguments in max-based semantics causes a similar effect.
%
%
The exponential complexity of Strategy 2 makes it unsuitable for large CAFs. Our results suggest that  for the weighted card-based and max-based semantics, the impact of using Strategy 1 is small, and that for the weighted h-categorizer semantics Strategy 1 performs well.


\section{Related work}
\label{sec:related_worl}

In \cite{DBLP:conf/nmr/EspinozaNT23}, the authors introduce credal support argumentation frameworks where each argument is associated with a credal set and an imprecise base score obtained from its credal set. While they only consider a support relation, they show how to compute the imprecise strength of arguments (as an interval) and study  theoretical properties.
The epistemic approach to probabilistic argumentation \cite{hunter21probabilistic} aims to determine valid probabilities for arguments given some properties, similar to our intervals. In the context of fuzzy argumentation, legal argument weights can be computed according to the approach of \cite{DBLP:conf/comma/WuLON16}. However, in all these cases, the properties of the interval are not explicitly considered, and neither is correcting systems which do not comply with the underlying properties.
Enforcement in abstract argumentation (e.g., \cite{baumann:hal-03541704}) considers how argumentation frameworks can be modified to guarantee an argument's status. Our refinements can be viewed similarly in a gradual semantics context.

\section{Conclusions and Future Work}
\label{sec:conclusion}

In this paper, we use gradual semantics to help align user beliefs as intervals of acceptability degrees for each argument.  
From an initial set of acceptability degree intervals, we introduced refinment  (i.e., tighten the bounds of the interval without losing any valid solutions) and rationalization (i.e., modify bounds to include valid solutions when these were not originally present). 
The latter was achieved via cost aware heuristics which, while not optimal, perform well empirically.

Our work utilizes weighted gradual semantics and intervals on argument acceptability degrees to elicit initial weights. This can enhance data collection in argumentation graphs, with applications to machine learning training; unlike current approaches that depend on synthetic data to represent initial weights, our approach offers a more realistic and potentially more accurate alternative \cite{DBLP:conf/icaart/AnaissySSV24}.

We lay the groundwork for several avenues of future work. First, we intend to investigate applications of our formalism. For example, the International Panel on Climate Change\footnote{\url{https://www.ipcc.ch/}} 
specifies intervals in their documents around the impacts of climate change (e.g., "very likely" means that there is a 90-95\% likelihood of an outcome occuring). Such intervals can serve as acceptability degrees intervals for conclusions, and encoding the data using our formalism would give insight into which semantics human use when reasoning in this context (by analyzing the semantics that minimizes irrationality). Human-based evaluation can provide insight here. 
Related to this, we plan to investigate how to sample the ``best'' valid initial weights, and present/explain them to a user, allowing them to identify those that they wish to use. 
We also wish to extend our formalism by extending intervals to various probability distributions  to better capture human reasoning. 

\bibliographystyle{unsrt}
\bibliography{mybibfile}

\end{document}